\documentclass{article}

\usepackage[preprint]{neurips_2026}

\usepackage[utf8]{inputenc}
\usepackage[T1]{fontenc}
\usepackage{microtype}
\usepackage{url}
\usepackage{booktabs}
\usepackage{amsmath,amssymb,amsfonts,amsthm,mathtools}
\usepackage{graphicx}
\usepackage{xcolor}
\usepackage{tikz}
\usepackage{enumitem}
\usepackage{hyperref}

\title{Topological Neural Tangent Kernel}

\author{%
  Sanjukta Krishnagopal \\
  Department of Computer Science, University of California Santa Barbara \\
  \texttt{sanjukta@ucsb.edu}
}

\newtheorem{theorem}{Theorem}
\newtheorem{proposition}{Proposition}

\newtheorem{definition}{Definition}

\begin{document}

\maketitle

\begin{abstract}
Graph neural tangent kernels give a principled infinite-width theory for graph
neural networks, but inherit a basic limitation of graph models: they use only
pairwise structure unless higher-order information is explicitly encoded in the
features. Many relational systems contain higher-order interactions that are
more naturally represented by simplicial complexes. We introduce TopoNTK, the
infinite-width kernel associated with Hodge message-passing on edge
features. TopoNTK combines lower Hodge interactions, capturing graph-like
coupling through shared vertices, with upper Hodge interactions, capturing
coupling through filled simplices. This makes the kernel sensitive to
filled-simplex structure invisible to graph-only kernels, allowing complexes
with the same graph and features but different filled simplices to induce
different edge-level kernels whenever the upper Hodge
operator changes.

Beyond expressivity, the Hodge structure gives an interpretable learning
geometry. Edge signals decompose into gradient-like, harmonic, and local
circulation components. The propagation operator exactly preserves the Hodge
decomposition; for nonlinear TopoNTKs, exact kernel-level invariance requires
explicit compatibility conditions, while the standard ReLU recursion gives a
Hodge-aligned architectural bias. Kernel spectra then diagnose which projected
topological components are learned quickly or slowly. We prove an expressivity separation,
propagation-level Hodge preservation, conditional kernel compatibility, spectral
learning, and finite-depth stability, and support these claims on synthetic
simplicial tasks and DBLP higher-order link prediction.
\end{abstract}

\section{Introduction}

Graph neural networks and their infinite-width limits, neural tangent kernels (NTKs), provide a principled framework for learning on relational data \cite{jacot2018ntk,du2019graph,krishnagopal2023gntk}. 
These methods are built around pairwise relations, where vertices represent objects and edges represent interactions. 
However, many complex systems involve higher-order relations among groups of entities, including multi-person social groups, biological assemblies, neural co-activity patterns, and multi-agent interactions \cite{battiston2020networks,krishnagopal2021spectral,krishnagopal2022mountaineering}. 
In such settings, the distinction between an unfilled cycle and a filled triangle (representing a single three-way simultaneous interaction) is not merely a modeling detail: it changes the topology, the geometry of edge signals, and the relevant notion of similarity.

Simplicial complexes provide a natural representation for this structure by extending graphs with higher-dimensional faces while retaining the graph as the $1$-skeleton. Here, a \(0\)-simplex is a node, a \(1\)-simplex is an edge, a \(2\)-simplex is a filled triangle, and more generally a \(k\)-simplex represents an interaction among \(k+1\) nodes. Thus, graphs are the simplest case containing only \(0\)- and \(1\)-simplices.
Simplicial complexes also support a Hodge-theoretic decomposition of edge signals into exact, harmonic, and coexact components, corresponding to gradient-like flows, cycle flows modulo filled boundaries, and local circulations around filled simplices (see Figure~\ref{fig:hodge_decomposition}). 
Hodge Laplacians and their lower--upper decomposition are standard tools for analyzing such edge-supported signals \cite{horak2013spectra,lim2020hodge,schaub2020random} with applications in many real-world social and biological systems \cite{su2024hodge}. 

Recent work in topological deep learning frames simplicial complexes, cell complexes,
and related domains as natural extensions of graph representation learning
\cite{hajij2022topological,papamarkou2024topological},  using lower and upper adjacency or Hodge-based message passing \cite{ebli2020simplicial,bodnar2021weisfeiler,yang2022simplicialfilters}. 
These works show that higher-order structure can provide information unavailable to graph-only models. 
However, the corresponding kernel perspective remains undeveloped. 
This raises a natural question: can the NTK framework be extended to simplicial complexes in a way that captures higher-order topology and provides insight on the learning dynamics?

We introduce the \emph{Topological Neural Tangent Kernel} (TopoNTK), the infinite-width kernel associated with a simplified Hodge message-passing architecture on edge features. 
TopoNTK propagates information through two complementary Hodge channels: a lower channel induced by shared vertices and an upper channel induced by shared filled simplices. 
The lower channel recovers graph-like interactions on the $1$-skeleton, while the upper channel captures filled-simplex structure invisible to graph kernels using only the same graph and input features. Thus, TopoNTK can distinguish simplicial complexes with identical underlying graphs but different triangle sets.

Our main observation is that this Hodge-theoretic structure gives both expressivity and interpretability. 
Graph NTKs, and lower-only topological kernels, cannot distinguish complexes with the same $1$-skeleton and identical features but different triangle sets. 
By contrast, the full TopoNTK depends on both lower and upper Hodge Laplacians and induces learning behavior aligned with exact, harmonic, and coexact edge-signal components. 
This yields a topologically structured spectral diagnostic with finite-depth Lipschitz stability under upper-Laplacian perturbations, useful when signals depend on group-level structure, cycles, or circulations rather than pairwise connectivity alone.

\textbf{Contributions} We make four contributions:
(i) we define TopoNTK, an infinite-width kernel associated with Hodge message passing
on edge features;
(ii) we show that upper Hodge propagation can capture filled-simplex structure
invisible to graph NTKs and lower-only variants using the same graph and features;
(iii) we prove exact Hodge preservation of propagation, give sufficient conditions for kernel-level Hodge compatibility, derive spectral learning dynamics, and prove finite-depth stability under upper-Laplacian perturbations; and
(iv) we empirically probe these predictions on synthetic and real data on tasks such as triangle-count sensitivity, Hodge component
recovery, and DBLP higher-order link
prediction.

\section{Background}

\subsection{Simplicial complexes and Hodge theory}

Let $X$ be an oriented simplicial complex with vertex set $V$, edge set $E$, and triangle set $T$. 
While a graph records only pairwise relations, a simplicial complex also records higher-order relations: a filled triangle $\{i,j,k\}\in T$ represents a simultaneous three-way interaction rather than merely the three pairwise edges among its vertices. 
This distinction is central: two complexes may have the same $1$-skeleton but different triangle sets, and hence different higher-order topology, that isn't captured when they are represented by their graph skeleton.

We focus on edge signals, or real-valued $1$-cochains,
\[
x \in C^1(X) \cong \mathbb{R}^{|E|}.
\]
We identify finite-dimensional chains and cochains with their Euclidean coordinate vectors. Let \(B_1\in\mathbb R^{|V|\times |E|}\) map oriented edges to vertices and \(B_2\in\mathbb R^{|E|\times |T|}\) map oriented triangles to edges. They satisfy the chain-complex identity
\[
B_1 B_2 = 0,
\]
which algebraically expresses that the boundary of a boundary is zero.

The Hodge $1$-Laplacian is
\[
L_1 = B_1^\top B_1 + B_2 B_2^\top.
\]
We write
\[
L_{\downarrow} := B_1^\top B_1,
\qquad
L_{\uparrow} := B_2 B_2^\top.
\]
The lower Laplacian $L_{\downarrow}$ couples edges that share vertices and is determined by the graph $1$-skeleton. 
The upper Laplacian $L_{\uparrow}$ couples edges that co-bound filled triangles and therefore depends on the higher-order simplicial structure. These operators are standard in combinatorial
Hodge theory and simplicial signal processing
\cite{horak2013spectra,lim2020hodge,schaub2020random,barbarossa2020topological}.

The Hodge decomposition gives the orthogonal splitting
\[
C^1(X)=\operatorname{im}B_1^\top \oplus \ker L_1 \oplus \operatorname{im}B_2 .
\]
We call these the exact, harmonic, and coexact subspaces. Exact components are gradient-like edge flows induced by vertex potentials, harmonic components are global cycle flows, and coexact components are local circulations around filled triangles. This decomposition is the natural coordinate system for topological learning on edge signals (see Figure~\ref{fig:hodge_decomposition}).
This interpretation underlies a large body of work on edge flows, circulation,
and higher-order signal processing on simplicial complexes
\cite{schaub2020random,lim2020hodge}.

More broadly, higher-order network methods support signals defined not only on nodes and edges but also on higher-dimensional \(k\)-simplices through the \(k\)-th Hodge Laplacian~\cite{schaub2021signal,bick2023higher}.

\subsection{Neural tangent kernels}
The neural tangent kernel (NTK) describes the infinite-width training dynamics of neural networks under gradient descent~\cite{jacot2018ntk}. For a network \(f_\theta\), the NTK is
\[
K(x,x')=\langle \nabla_\theta f_\theta(x),\nabla_\theta f_\theta(x')\rangle .
\]
In the infinite-width limit, many architectures induce deterministic kernels, and squared-loss training is governed by kernel gradient flow or, with ridge regularization, kernel ridge regression. Given a training Gram matrix \(K\) and labels $y$, kernel ridge regression has fitted training values
\[
\hat y = K(K+\lambda I)^{-1}y,
\]
and test prediction \(k_*^\top(K+\lambda I)^{-1}y\). Thus the eigenspaces of \(K\) determine the learning bias: large-eigenvalue target components are fit more strongly than small-eigenvalue components.

\textbf{Graph NTKs:} Graph NTKs arise from message-passing graph neural networks \citep{du2019graph,krishnagopal2023gntk}. A typical message-passing layer has the form
$
H^{(\ell+1)}=\sigma\!\left(\widetilde A H^{(\ell)}W^{(\ell)}\right),$
where $\widetilde{A}$ is a graph propagation operator, $H^{(\ell)}$ is the hidden representation at layer $\ell$, $W^{(\ell)}$ is a trainable weight matrix, and $\sigma$ is a nonlinear activation function. 
In the infinite-width limit, the associated covariance and tangent kernels obey deterministic recursions determined by $\widetilde{A}$ and $\sigma$. Since their propagation operators are graph-derived, they depend only on the \(1\)-skeleton. Therefore, if two simplicial complexes have the same graph and the same input features but different filled triangles, a graph NTK using only the 1-skeleton assigns them the same kernel. This motivates a topological NTK that incorporates upper Hodge structure.

\begin{figure}[t]
    \centering
    \includegraphics[width=0.65\linewidth]{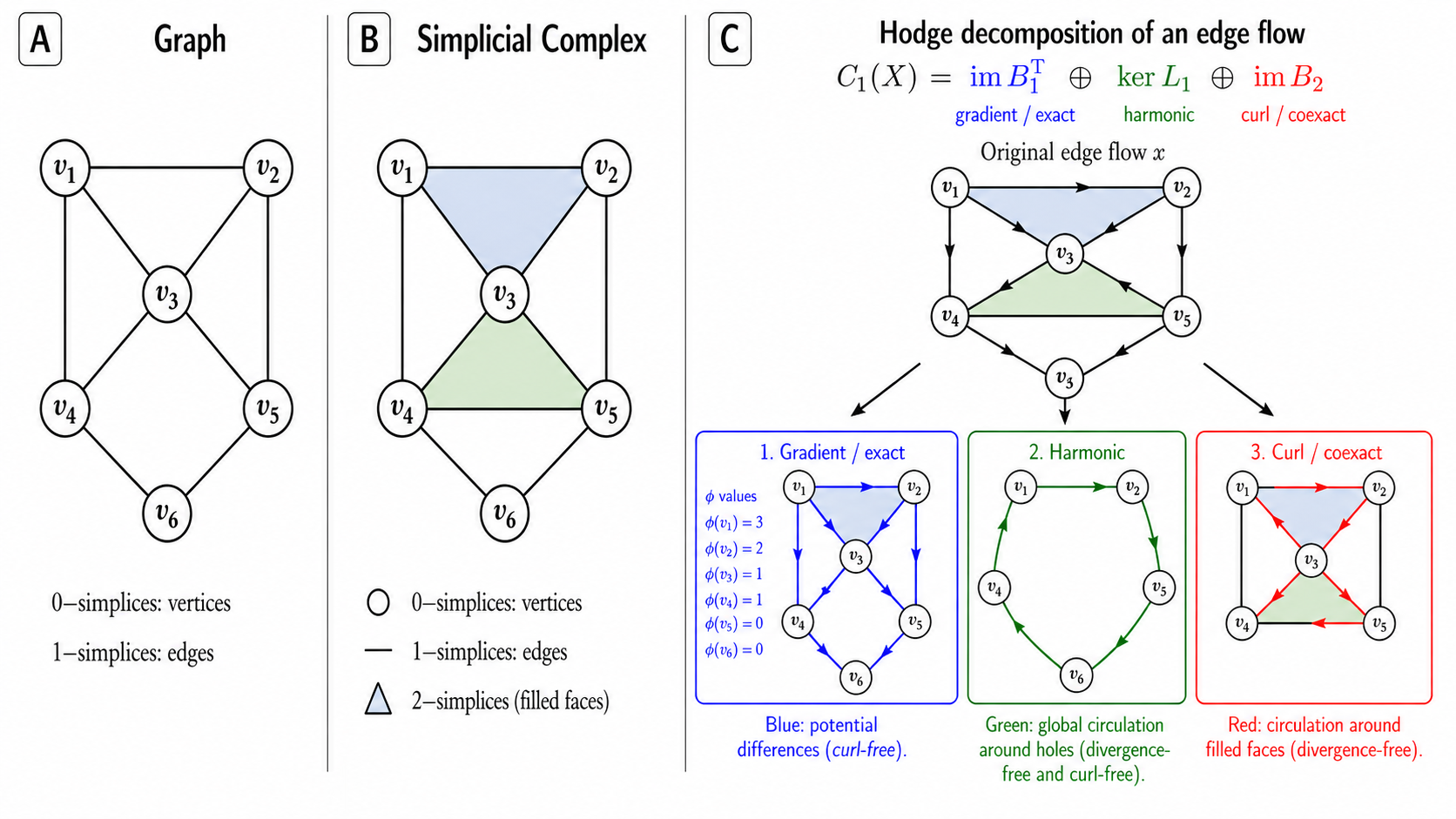}
\caption{
    \textbf{Graphs, simplicial complexes, and Hodge decomposition.}
    A simplicial complex augments a graph with filled higher-order interactions. Edge signals decompose into exact, harmonic, and coexact components, corresponding to gradient-like flows, global cycles modulo filled boundaries, and local circulations.
    }
\label{fig:hodge_decomposition}
\end{figure}

\section{Topological Neural Tangent Kernels}

\subsection{Simplicial message passing}

Graph message passing propagates information through shared vertices. 
On simplicial complexes, interactions also occur through higher-order structure such as filled triangles. 
We model this by combining two propagation mechanisms for edge signals: a lower channel capturing graph-like interactions (edge signals propagating through a shared node) and an upper channel capturing higher-order interactions (edge signals propagating via being faces of a shared triangle). This lower--upper propagation viewpoint is closely related to simplicial neural
networks and Hodge-based convolutional filters
\cite{ebli2020simplicial,bodnar2021weisfeiler,yang2022simplicialfilters}.

Let $X$ be an oriented simplicial complex and let
\[
H^{(0)} \in \mathbb{R}^{|E| \times d}
\]
denote $d$-dimensional edge features. Define the Hodge propagation operator
\[
P_{\gamma,\alpha,\beta}
=
\gamma I+\alpha L_\downarrow+\beta L_\uparrow,
\]
where 
\[
L_{\downarrow}=B_1^\top B_1,
\qquad
L_{\uparrow}=B_2B_2^\top .
\]
The lower term couples edges sharing vertices; the upper term couples edges that co-bound a filled triangle.
The residual term acts as an edge self-loop and preserves harmonic signals,
since \(L_\downarrow h=L_\uparrow h=0\) for \(h\in\ker L_1\).

A simplicial message-passing layer, written in the ordering used by the kernel recursion below, is
\[
H^{(\ell+1)}
=
P_{\gamma,\alpha,\beta}\,\sigma\!\left(H^{(\ell)}W^{(\ell)}\right),
\]
where $W^{(\ell)}\in\mathbb R^{m_\ell\times m_{\ell+1}}$ has entries \(N(0,1/m_\ell)\) at initialization and $\sigma$ is applied entrywise. The lower, upper, and full variants use
\(\gamma I+\alpha L_\downarrow\), \(\gamma I+\beta L_\uparrow\), and
\(\gamma I+\alpha L_\downarrow+\beta L_\uparrow\), respectively. Appendix~\ref{app:ntk_derivation} derives the corresponding infinite-width recursion. A pre-activation convention, \(\sigma(PH^{(\ell)}W^{(\ell)})\), gives a related kernel with \(P\) inside the activation covariance map. The same construction can be formulated for \(k\)-cochains using \(L_k^\downarrow,L_k^\uparrow\).

\subsection{The Topological Neural Tangent Kernel}

The neural tangent kernel captures the infinite-width limit of a neural network. 
For simplicial message passing, this yields a kernel that depends on both graph structure and higher-order topology through the lower and upper Hodge operators.

\begin{definition}[Topological Neural Tangent Kernel]
\label{def:topontk}
Let $X$ and $Y$ be simplicial complexes with edge features $H_X^{(0)}, H_Y^{(0)}$. Define
\[
P_X=\gamma I+\alpha L_\downarrow(X)+\beta L_\uparrow(X),
\qquad
P_Y=\gamma I+\alpha L_\downarrow(Y)+\beta L_\uparrow(Y).
\]
Initialize
\[
\Sigma^{(0)}(X,Y)=\frac{1}{d}H_X^{(0)}(H_Y^{(0)})^\top.
\]
We set $\Theta^{(0)}(X,Y)=\Sigma^{(0)}(X,Y)$. For $\ell\geq 0$,
\[
\Sigma^{(\ell+1)}(X,Y)
=
P_X\,\Phi(\Sigma^{(\ell)}(X,Y))P_Y^\top,
\]
\[
\Theta^{(\ell+1)}(X,Y)
=
P_X\!\left[
\Theta^{(\ell)}(X,Y)\odot
\dot\Phi(\Sigma^{(\ell)}(X,Y))
\right]P_Y^\top
+
\Sigma^{(\ell+1)}(X,Y).
\]

\end{definition}
For a nonlinearity $\sigma$, the notation
\(\Phi(\Sigma^{(\ell)}(X,Y))\) abbreviates the standard dual-activation
covariance map evaluated entrywise using the cross-covariance
\(\Sigma^{(\ell)}(X,Y)\) together with the corresponding self-covariance
diagonals from \(\Sigma^{(\ell)}(X,X)\) and \(\Sigma^{(\ell)}(Y,Y)\).
The derivative map \(\dot\Phi\) is defined analogously. Thus, entrywise,
\[
\Phi(\Sigma^{(\ell)})_{ij}
=
\mathbb E[\sigma(g_i^X)\sigma(g_j^Y)],
\qquad
\dot\Phi(\Sigma^{(\ell)})_{ij}
=
\mathbb E[\sigma'(g_i^X)\sigma'(g_j^Y)],
\]
where \((g_i^X,g_j^Y)\) is centered Gaussian with variances
\(\Sigma^{(\ell)}_{ii}(X,X)\), \(\Sigma^{(\ell)}_{jj}(Y,Y)\), and covariance
\(\Sigma^{(\ell)}_{ij}(X,Y)\).

For graph-level or complex-level prediction, we use the pooled scalar kernel
\[
K(X,Y)=\mathbf{1}_{E_X}^{\top}\Theta^{(L)}(X,Y)\mathbf{1}_{E_Y},
\]
optionally normalized by $|E_X|^{-1/2}|E_Y|^{-1/2}$.
For a fixed complex \(X\), we write
$
K_X:=\Theta^{(L)}(X,X),
$
an edge-by-edge kernel matrix acting on edge signals \(x\in C^1(X)\cong \mathbb R^{|E|}\).

Unlike graph NTKs using only the same 1-skeleton and features, TopoNTK incorporates $L_{\uparrow}$ and is therefore sensitive to filled triangles.

\section{Theoretical Properties}
\label{sec:theory}

\subsection{Expressivity beyond graph kernels}
\label{sec:expressivity}
Graph kernels using only pairwise structure cannot distinguish simplicial
complexes that share the same $1$-skeleton and input features. In contrast,
TopoNTK incorporates the upper Hodge Laplacian and can detect differences in
filled-simplex structure.

\begin{proposition}[Filled-simplex sensitivity]
\label{prop:filled-simplex-sensitivity}
Let \(X\) and \(X'\) be simplicial complexes with the same oriented
\(1\)-skeleton. Any graph NTK whose propagation and input features depend only
on the \(1\)-skeleton assigns the same kernel to \(X\) and \(X'\).

By contrast, let
\[
P_X=\gamma I+\alpha L_\downarrow+\beta L_\uparrow(X),
\qquad
P_{X'}=\gamma I+\alpha L_\downarrow+\beta L_\uparrow(X').
\]
If \(\beta>0\), \(L_\uparrow(X)\neq L_\uparrow(X')\), and there exists a
positive semidefinite edge covariance \(C\) such that
\[
P_X C P_X^\top \neq P_{X'} C P_{X'}^\top,
\]
then there are edge features, realizing \(C=\frac{1}{d}HH^\top\), for which the
edge-level TopoNTK differs after one layer. A pooled scalar kernel separates the
pair whenever the pooling functional is nonzero on this edge-level kernel
difference.
\end{proposition}

The proof is given in Appendix~\ref{app:expressivity}. The condition says only that the changed filled simplices must be visible to the
chosen edge-feature covariance and pooling. Thus the result is an expressivity
separation: graph-only kernels are invariant to triangle changes on a fixed
\(1\)-skeleton, whereas TopoNTK can respond through \(L_\uparrow\).

\subsection{Hodge-aligned structure}
\label{sec:hodge}

Edge signals admit a decomposition into exact, harmonic, and coexact components.
The lower and upper Hodge operators do not mix these types, which is why the TopoNTK learns edge-signal components in coordinates adapted to the Hodge decomposition.

\begin{proposition}[Hodge preservation of propagation]
\label{prop:hodge}
The residual Hodge propagation operator
\(P_{\gamma,\alpha,\beta}=\gamma I+\alpha L_\downarrow+\beta L_\uparrow\)
preserves the Hodge decomposition:
\[
P_{\gamma,\alpha,\beta}(\mathcal E)\subseteq\mathcal E,\quad
P_{\gamma,\alpha,\beta}(\mathcal H)\subseteq\mathcal H,\quad
P_{\gamma,\alpha,\beta}(\mathcal C)\subseteq\mathcal C.
\]
Moreover, \(L_\uparrow\) vanishes on \(\mathcal E\), \(L_\downarrow\) vanishes
on \(\mathcal C\), both vanish on \(\mathcal H\), and
\(P_{\gamma,\alpha,\beta}h=\gamma h\) for \(h\in\mathcal H\).
\end{proposition}

\begin{proof}
If \(x=B_1^\top u\in\mathcal E\), then \(L_\uparrow x=B_2(B_1B_2)^\top u=0\). If \(x=B_2v\in\mathcal C\), then \(L_\downarrow x=B_1^\top B_1B_2v=0\). If \(h\in\mathcal H\), then \(B_1h=0\) and \(B_2^\top h=0\), so both Hodge Laplacian terms vanish. The identity term preserves every subspace, giving the claim.
\end{proof}

\paragraph{Kernel-level Hodge compatibility.}
Proposition~\ref{prop:hodge} is the main structural fact: the Hodge propagator itself does not mix exact, harmonic, and coexact signals. Kernel-level invariance additionally requires the covariance and tangent-kernel updates to preserve the same blocks. Useful sufficient conditions are: (i) linear activations together with Hodge-block-compatible initial covariance; (ii) Hodge-filter kernels whose covariances and tangent kernels are matrix functions or finite sums of products of \(L_\downarrow\) and \(L_\uparrow\); or (iii) any recursion for which \(\Sigma^{(0)}\), \(\Phi(\Sigma^{(\ell)})\), and \(\dot\Phi(\Sigma^{(\ell)})\) remain block diagonal in
\[
C^1(X)=\mathcal E\oplus\mathcal H\oplus\mathcal C.
\]
For nonlinear ReLU NTK recursions with entrywise Hadamard products, this compatibility is not automatic, because the Hodge decomposition is generally not aligned with the standard edge basis. Thus the nonlinear TopoNTK should be understood as a Hodge-aligned architectural bias, while exact block invariance holds under the compatibility condition below.

\begin{proposition}[Sufficient condition for kernel-level Hodge compatibility]
\label{prop:hodge_compatibility}
Assume \(\Sigma^{(0)}(X,X)\) is block diagonal with respect to
\(C^1(X)=\mathcal E\oplus\mathcal H\oplus\mathcal C\), and assume that, for each
finite layer, \(\Phi(\Sigma^{(\ell)}(X,X))\) and
\(\dot\Phi(\Sigma^{(\ell)}(X,X))\) are block diagonal in the same decomposition.
Then the within-complex TopoNTK preserves the Hodge decomposition:
\[
K_X(\mathcal{E})\subseteq \mathcal{E},\qquad
K_X(\mathcal{H})\subseteq \mathcal{H},\qquad
K_X(\mathcal{C})\subseteq \mathcal{C}.
\]
The sufficient conditions above are examples where these hypotheses hold.
\end{proposition}

See Appendix~\ref{app:hodge_invariance_NTK}. The proposition records when propagation-level Hodge preservation lifts to the full kernel; the experiments below test whether the standard ReLU recursion retains this Hodge-aligned bias in practice.

\subsection{Spectral learning dynamics}
\label{sec:spectral}

Kernel methods learn along eigenvectors, with rates determined by eigenvalues. When the Hodge-compatibility assumptions hold, \(K_X\) is block diagonal with respect to the Hodge decomposition, so eigenvectors may be chosen within the exact, harmonic, and coexact subspaces. More generally, one can project eigenvectors onto these subspaces to diagnose how the nonlinear TopoNTK distributes Hodge components across its spectrum. Components aligned with larger kernel
eigenvalues are learned more quickly, while components aligned with smaller
eigenvalues are learned more slowly.

\begin{theorem}[Spectral learning under kernel gradient flow]
\label{th:spectral}
Let \(K_X\) be the TopoNTK, and let
$
K_X u_j=\kappa_j u_j
$
be an orthonormal eigenbasis of \(K_X\). Suppose kernel gradient flow is initialized at
\(f_0=0\) and trained to fit a target edge signal \(y\). Then the prediction at time \(t\) is
\[
f_t
=
\sum_j
\left(1-e^{-\kappa_j t}\right)
\langle y,u_j\rangle u_j .
\]
Thus each kernel eigenmode is learned independently, and the learning rate of mode \(u_j\)
is determined by its eigenvalue \(\kappa_j\).
\end{theorem}

The proof is the standard eigenbasis solution of kernel gradient flow and is
given in Appendix~\ref{app:spectral}.

\subsection{Stability under higher-order perturbations}
\label{sec:stability}
If triangle structure is perturbed on a fixed oriented edge set, the upper Hodge
Laplacian changes and the finite-depth TopoNTK changes continuously with it.
\begin{theorem}
\label{th:stability}
Let \(X,X'\) be complexes on the same oriented edge set, differing only in their
triangle sets. Assume the finite-depth recursion is uniformly bounded and that
\(\Phi,\dot\Phi\) are Lipschitz on the relevant bounded covariance set. For ReLU,
assume additionally that covariance diagonals stay bounded away from zero. Then
there are constants \(C_K,C_{\rm pred}\), depending on the finite-depth bounds and
on \(L,\alpha,\beta,\gamma\), such that
\[
\|K_X-K_{X'}\|_F \le C_K\|L_\uparrow-L_\uparrow'\|_F .
\]
Moreover, for kernel ridge regression with ridge parameter \(\lambda>0\),
\[
\|\hat{y}_X - \hat{y}_{X'}\|
\le \frac{C_{\rm pred}}{\lambda}\|L_\uparrow-L_\uparrow'\|_F\|y\|.
\]
If \(B_2,B_2'\) are represented over a common oriented candidate-triangle set and
\(\|B_2\|,\|B'_2\|\le M\), then
\(\|L_\uparrow-L_\uparrow'\|\le 2M\|B_2-B_2'\|\).
\end{theorem}
The first bound is stated in Frobenius norm; the ridge bound uses operator-norm
resolvent estimates together with \(\|A\|_{\rm op}\le \|A\|_F\). See
Appendix~\ref{app:stability}.

\section{Experiments}
\label{sec:experiments}
We evaluate TopoNTK on synthetic topology-sensitive tasks and DBLP coauthorship
prediction. Baselines are Graph NTK, Lower TopoNTK, Upper TopoNTK, and Full
TopoNTK. Unless stated otherwise, we use kernel ridge regression with ReLU
activation covariance maps, depth $L=2$, $\gamma=0.5$, $\alpha=\beta=1$, scalar
spectral normalization of Hodge propagators, and trace normalization of the final
edge kernel. In fixed-complex edge-signal experiments, we use the architecture
operator induced by the recursion with $\Sigma^{(0)}=I$ and evaluate
$k(x,x')=x^\top K_{\rm arch}x'$.

\subsection{Higher-order prediction tasks}
\paragraph{Triangle-count prediction.}
We first test whether upper-dimensional information improves prediction of a
simplex-dependent target. The fixed 1-skeleton has $n=30$ vertices with cycle
edges $(i,i+1)$ and second-neighbor chords $(i,i+2)$; the $n$ triples
$(i,i+1,i+2)$ are candidate 2-simplices. For each density $q$, each candidate is
filled independently with probability $q$, and the target is the raw number of
filled triangles. We use 200 samples per density, 5 repetitions, a 70/30
train/test split, and $\lambda=10^{-4}$. Error bars are standard errors.

\paragraph{DBLP simplicial closure.}
We evaluate future three-author collaboration prediction on the ScHoLP
coauth-DBLP temporal network \citep{benson2018simplicial}, constructed from DBLP
metadata \citep{dblp}. After sorting by time, the first 70\% of simplices form
the historical complex and the remaining 30\% define future collaborations.
Candidates are historical closed triads, i.e. triples whose three pairwise
coauthorship edges appear in the historical \(1\)-skeleton but whose
three-author simplex is absent from history. Positives are candidates that
appear as three-author simplices in the future window; negatives are candidates
absent from both history and future.  We use local
ego complexes around candidate triples. The graph baseline is a node-level graph
NTK on the ego 1-skeleton, while TopoNTK variants are edge-level kernels using
$P=\gamma I+\alpha L_\downarrow+\beta L_\uparrow$ with scalar spectral
normalization of $L_\downarrow$ and $L_\uparrow$. We use 5 runs with at most
50,000 simplices, 120 positives and 120 negatives, ego size 10, a 70/30 split,
depth $L=2$, $\lambda=10^{-3}$, maximum simplex size 10, maximum group size 8,
and at most 200 local triangles.

\begin{figure}[t]
\centering
\begin{minipage}[t]{0.49\linewidth}
    \centering
    \includegraphics[width=\linewidth]{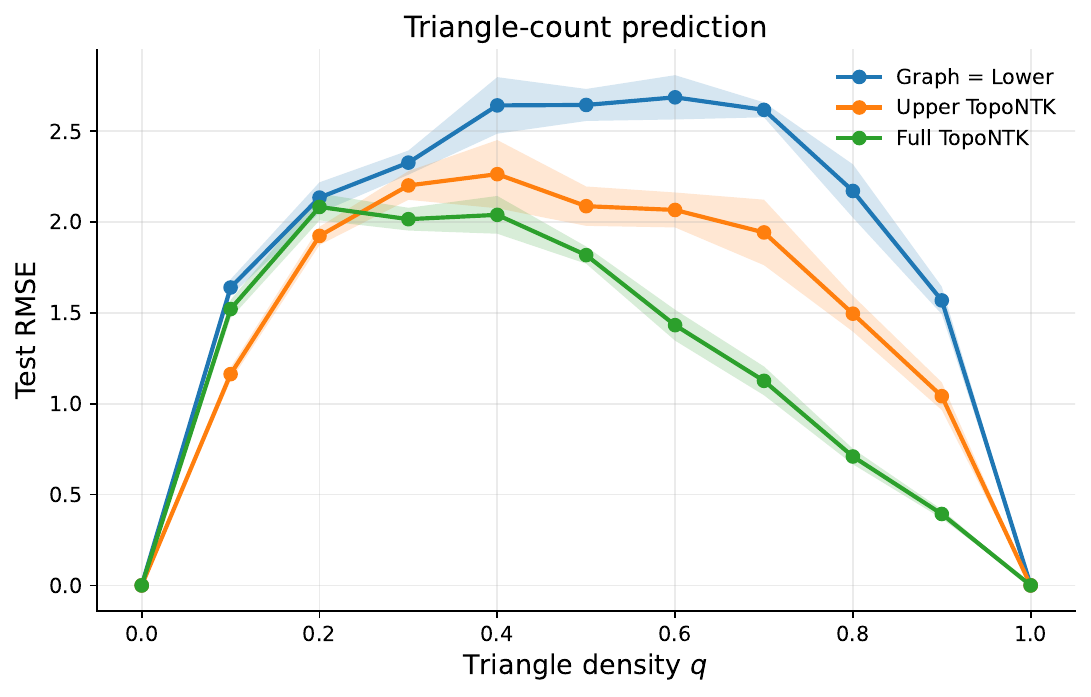}
\end{minipage}
\hfill
\begin{minipage}[t]{0.49\linewidth}
    \centering
    \includegraphics[width=\linewidth]{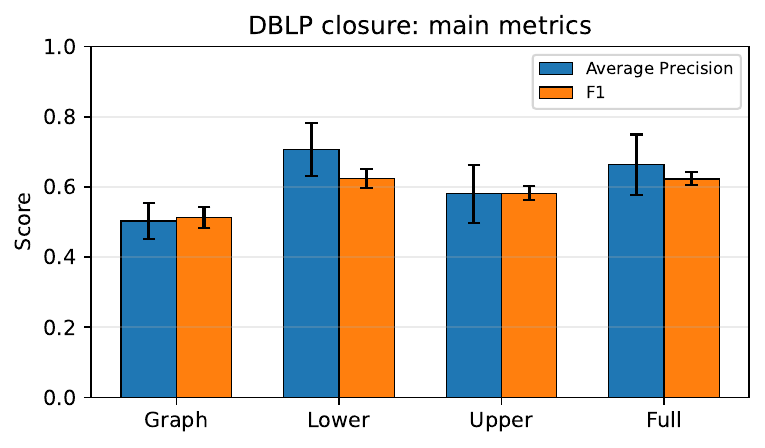}
\end{minipage}
\caption{
\textbf{Higher-order prediction in synthetic and real data.}
\textbf{Left:} Controlled triangle-count sensitivity test. Error bars denote standard error over \(5\) independent runs. Kernels using \(L_{\uparrow}\) outperform the graph/lower baseline, which is sample-invariant in this fixed-skeleton task.
\textbf{Right:} Future three-author collaboration closure on DBLP. We report average precision and F1, with error bars denoting standard error over \(5\) runs. Upper and full TopoNTK-style variants achieve the strongest performance, indicating that filled-triangle collaboration history contains predictive information beyond the pairwise graph.
}
\label{fig:higher-order-prediction}
\end{figure}

Figure~\ref{fig:higher-order-prediction} shows the same behavior in controlled and real settings. In triangle-count regression, Upper TopoNTK and Full TopoNTK achieve substantially lower error than graph-only/lower-only baselines, as expected for a task affected by filled \(2\)-simplices. $q=0,1$ are deterministic and therefore trivial. This experiment is intended as a controlled sensitivity diagnostic rather than a
difficult predictive benchmark: the target is deliberately a function of the
filled \(2\)-simplices. The graph/lower kernel contains no sample-specific
higher-order information in this fixed-skeleton setting. Here, Graph NTK and Lower TopoNTK are identical, so we plot them as a single Graph = Lower curve.
On DBLP, the graph NTK and lower TopoNTK remain predictive, via average precision and F1, because pairwise coauthorship and shared-author structure still provide strong information about future collaborations. The gains from the upper and full channel reflect additional predictive signal from higher-order collaborations in this normalized local implementation. Note that full isn't always better than upper; lower channel may introduce pairwise noise for purely higher-order tasks.

\subsection{Hodge recovery and spectral diagnostic}
We next test whether the kernels align with the Hodge decomposition. For each
random complex, we sample an Erd\H{o}s--R\'enyi graph $G(n,p)$ with $n=20$,
$p=0.35$, fill each 3-clique with probability $q=0.4$, and sample unit-norm edge
signals by summing independent exact, harmonic, and coexact components. The
sample Gram matrix is $G_{ij}=x_i^\top K_{\rm arch}x_j$, and multi-output kernel
ridge regression recovers each target component: exact, harmonic, and coexact. We use 120 training signals, 60
test signals, 5 seeds, depth 2, and $\lambda=10^{-3}$.

For the spectral diagnostic, we compute eigenpairs
$K_{\rm full}u_j=\kappa_j u_j$ on an ER complex with $n=30$, $p=0.35$, $q=0.4$.
For a target $y=u_j$, Theorem~1 predicts
$\|f_t-u_j\|=e^{-\kappa_j t}\|u_j\|$; the diagnostic content is therefore the
placement of Hodge-labeled modes across the spectrum.

Figure~\ref{fig:hodge-spectral} shows that lower propagation improves exact-component recovery, while
upper propagation improves coexact-component recovery. Harmonic components lie
in $\ker L_1$ and are not directly amplified by either Hodge channel; they are
retained through the identity channel and learned according to their kernel
eigenvalues. The spectral-bias plot on the right visualizes the kernel-gradient-flow law for the finite-dimensional TopoNTK operator. The decay rate is determined analytically by \(\kappa_j\) so agreement with the diagonal is a sanity check, and the distribution of eigenvalues is the diagnostic result. Interestingly, in this random-complex example, harmonic modes concentrate at
smaller eigenvalues, indicating slower predicted learning of global topological
structure.

\begin{figure}[t]
\centering
\begin{minipage}[t]{0.49\linewidth}
    \centering    \includegraphics[width=\linewidth]{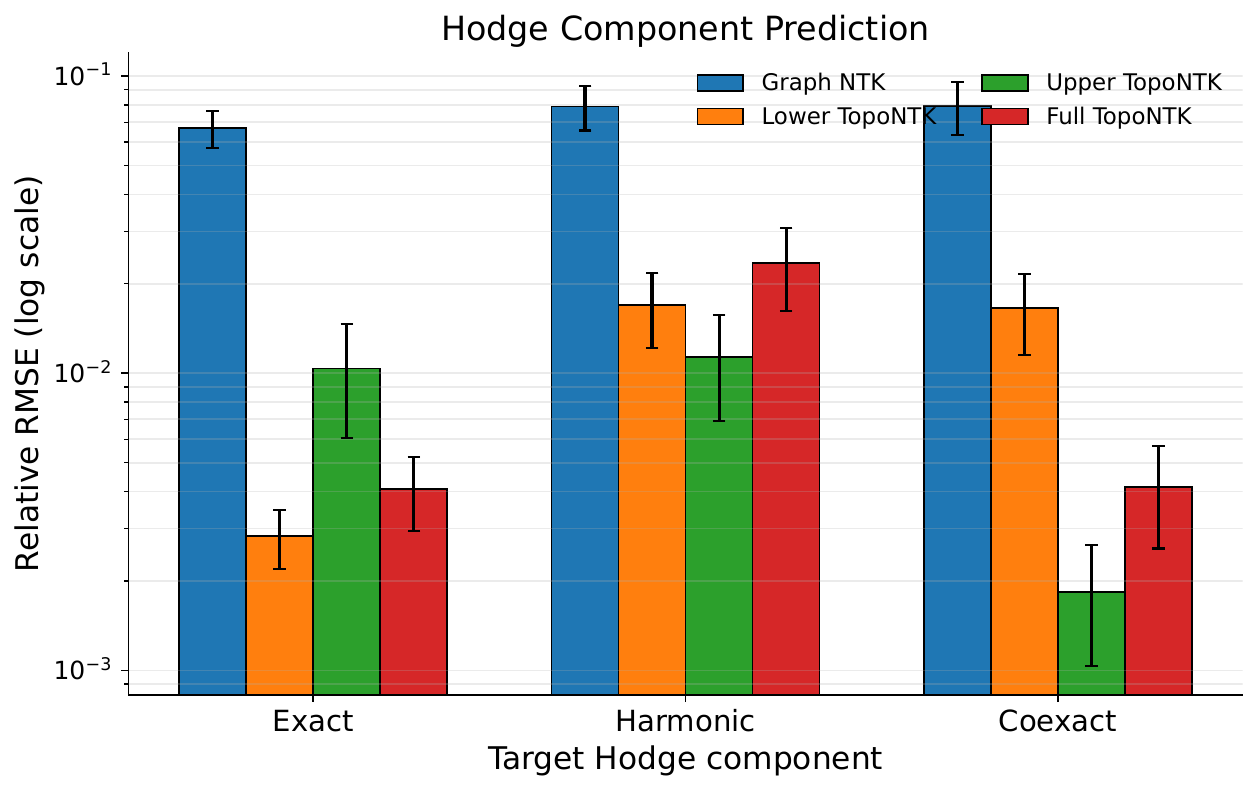}
\end{minipage}
\hfill
\begin{minipage}[t]{0.5\linewidth}
    \centering
    \includegraphics[width=\linewidth]{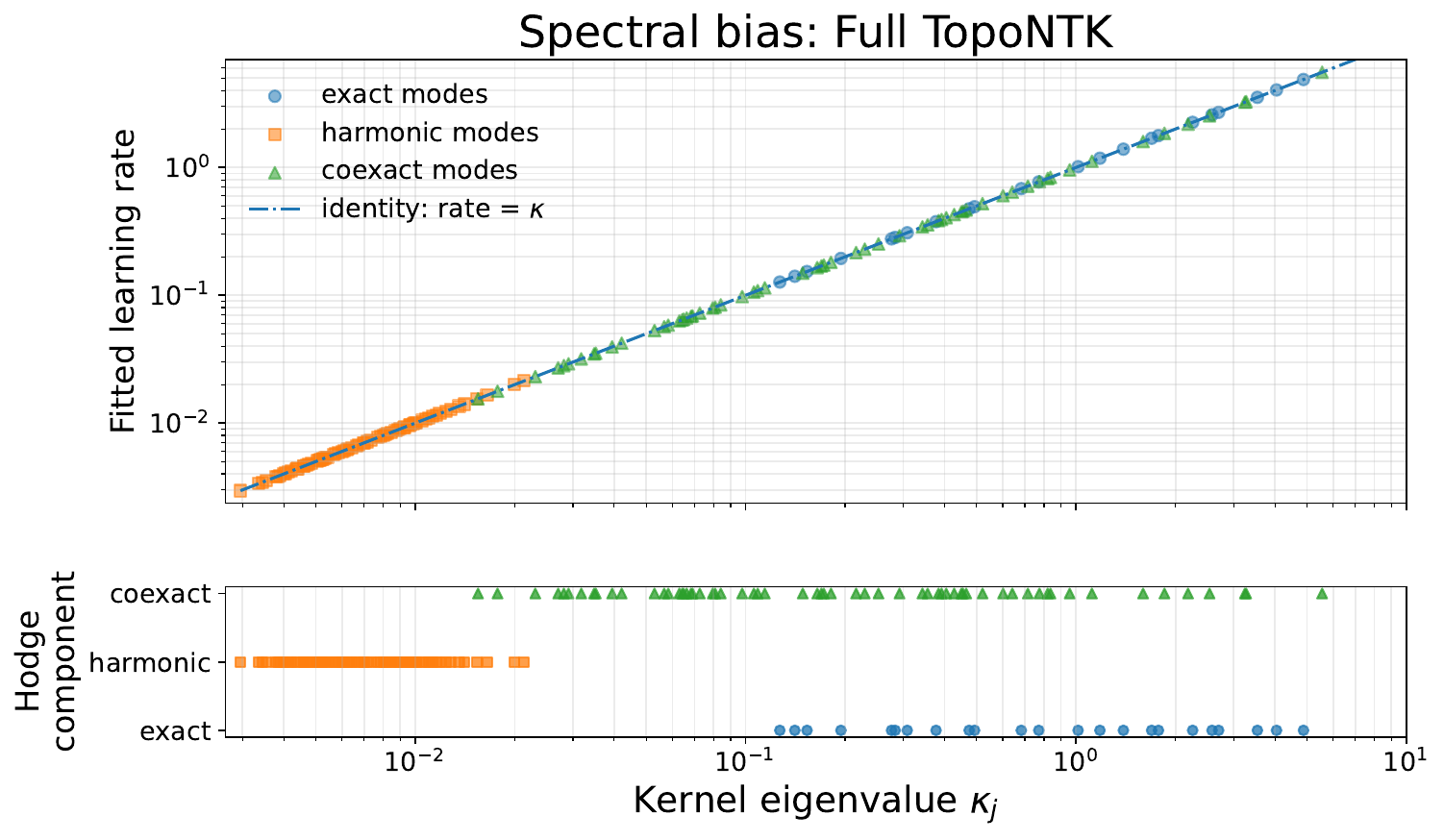}
\end{minipage}
\caption{
\textbf{Hodge-aligned recovery and spectral diagnostic.}
\textbf{Left:} Relative RMSE
for recovering exact, harmonic, and coexact edge-signal components. Lower
propagation improves exact-component recovery, while upper propagation improves
coexact-component recovery; harmonic components reflect global cycle structure.
Error bars denote standard error over \(5\) runs.
\textbf{Right:} Spectral diagnostic for the full TopoNTK. Each point is a kernel eigenmode, and the diagonal is the kernel-gradient-flow prediction from Theorem~\ref{th:spectral}. The main information is the placement of Hodge-labeled modes across the spectrum: in this random-complex example, harmonic modes concentrate at smaller eigenvalues and are predicted to learn more slowly.
}
\label{fig:hodge-spectral}
\end{figure}

\subsection{Stability under higher-order perturbations}

Finally, we test the stability result of Section~\ref{sec:stability}. Starting from a fixed \(1\)-skeleton $G(n,p)$ with \(n=30\), \(p=0.35\), and \(q=0.4\), we generate a complex \(X\) and construct perturbed complexes \(X'\) by independently flipping candidate triangles with probability \(\varepsilon\). Thus, the graph structure is unchanged while only the higher-order structure is corrupted. For each \(\varepsilon\), we use \(3\) independent perturbations per run, averaged over 5 runs. We use the full TopoNTK and compare ridge
parameters.

For each pair \((X,X')\), we compute
\[
\Delta_K
=
\frac{\|K_X-K_{X'}\|_F}{\|K_X\|_F},
\qquad
\Delta_L
=
\|L_{\uparrow}-L_{\uparrow}'\|_F.
\]
We also measure prediction stability under kernel ridge regression:
\[
\Delta_y
=
\frac{\|\hat y_X-\hat y_{X'}\|}{\|\hat y_X\|},
\qquad
\hat y_X
=
K_X(K_X+\lambda I)^{-1}y.
\]
In the code, \(y\) is a fixed mean-centered, unit-norm mixed edge signal on the base complex. The normalization in \(\Delta_y\) is empirical; the theorem gives an absolute bound proportional to \(\|y\|\). We vary both the triangle flip probability \(\varepsilon\) and the ridge parameter \(\lambda\).

\begin{figure}[t]
    \centering
    \includegraphics[width=\linewidth]{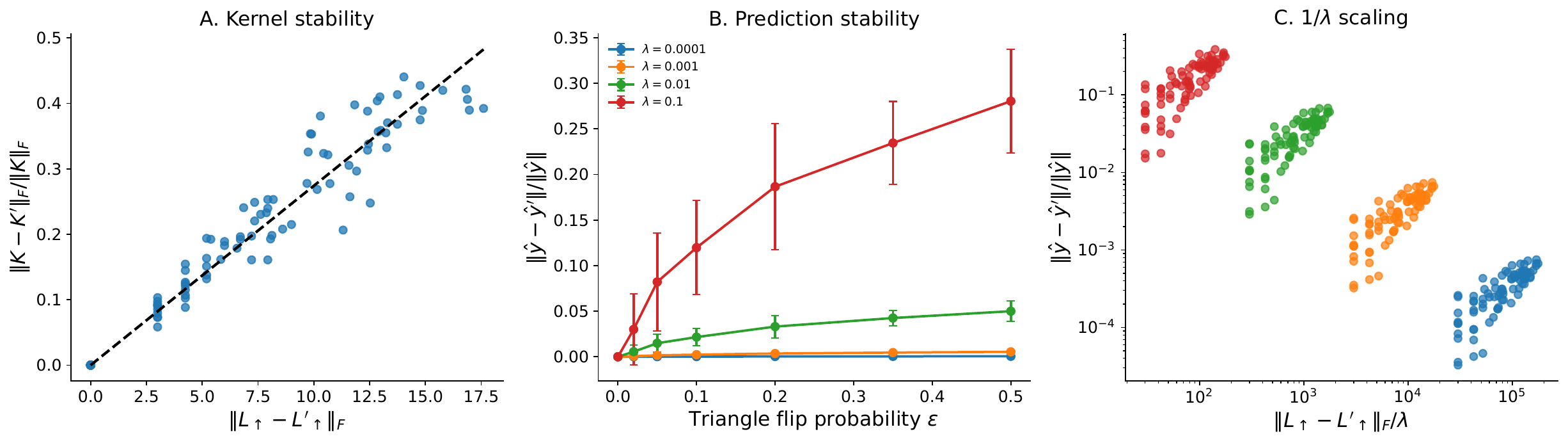}
    \caption{
    \textbf{Stability under triangle perturbations.}
    (A) Relative kernel perturbation versus change in the upper Laplacian. The approximately linear trend supports Lipschitz dependence on higher-order structure.
    (B) Prediction perturbation versus triangle flip probability \(\varepsilon\) for different ridge parameters \(\lambda\); stronger regularization improves stability.
    (C) Relative prediction errors plotted against \(\|L_{\uparrow}-L_{\uparrow}'\|_F/\lambda\), consistent with the predicted \(1/\lambda\)-type scaling from Section~\ref{sec:stability}.
    }
    \label{fig:stability-topontk}
\end{figure}

Figure~\ref{fig:stability-topontk} shows that TopoNTK varies stably under triangle perturbations. Kernel changes scale approximately linearly with changes in \(L_{\uparrow}\), while prediction changes grow smoothly with perturbation strength and decrease with stronger regularization. This supports the theory and is practically useful for error estimation and generalization when higher-order relations are noisy or partially observed.

\section{Conclusion}
\label{sec:conclusion}
We introduced TopoNTK, the infinite-width kernel induced by Hodge message passing on edge features. By incorporating lower and upper Hodge interactions, TopoNTK captures higher-order structure invisible to graph-only methods and gives a spectral view of exact, harmonic, and coexact components. Exact Hodge preservation holds for the propagation operator and for Hodge-compatible kernels; the standard nonlinear ReLU TopoNTK instead provides a strong architectural alignment. This is the main message: the upper channel gives expressivity beyond the 1-skeleton, the residual channel preserves harmonic information at propagation level, and the resulting spectra provide interpretable coordinates for relational learning with group interactions, cycles, and circulations.

\paragraph{Limitations and broader impact.} Exact kernel methods scale quadratically in sample size, and dense $n$-vertex complexes have $O(n^3)$ possible 2-simplices. The DBLP experiment therefore uses sampled local ego complexes and should be interpreted as evidence of a higher-order signal rather than a definitive large-scale benchmark. We compare kernel variants rather than trained finite-width SNN or MPSN models; scalable local-triangle approximations and finite-width comparisons are natural next steps. The method may benefit domains with group interactions, flows, or circulations, but in sensitive applications its increased expressivity may amplify biases present in data.

\paragraph{Code availability and compute.}
Code is available at \url{github.com/chimeraki/TopoNTK}. We used a workstation with Intel i7-10750H CPU, 16GB RAM, and
an NVIDIA GTX 1650 Ti GPU; runtimes were 1--5min for synthetic experiments
and 10--15min for DBLP.

\bibliography{references}

@article{horak2013spectra,
  title={Spectra of Combinatorial Laplace Operators on Simplicial Complexes},
  author={Horak, Danijela and Jost, J{\"u}rgen},
  journal={Advances in Mathematics},
  volume={244},
  pages={303--336},
  year={2013}
}

@article{lim2020hodge,
  title={Hodge Laplacians on Graphs},
  author={Lim, Lek-Heng},
  journal={SIAM Review},
  volume={62},
  number={3},
  pages={685--715},
  year={2020},
  publisher={SIAM},
  doi={10.1137/18M1223101}
}

@inproceedings{jacot2018ntk,
  title={Neural Tangent Kernel: Convergence and Generalization in Neural Networks},
  author={Jacot, Arthur and Gabriel, Franck and Hongler, Cl{\'e}ment},
  booktitle={Advances in Neural Information Processing Systems},
  year={2018}
}

@inproceedings{du2019graph,
  title={Graph Neural Tangent Kernel: Fusing Graph Neural Networks with Graph Kernels},
  author={Du, Simon S. and Hou, Karthik and Salakhutdinov, Ruslan and Poczos, Barnabas and Wang, Jianfeng},
  booktitle={Advances in Neural Information Processing Systems},
  year={2019}
}

@inproceedings{krishnagopal2023gntk,
  title={Graph Neural Tangent Kernel: Convergence on Large Graphs},
  author={Krishnagopal, Sanjukta and Ruiz, Luana},
  booktitle={Proceedings of the 40th International Conference on Machine Learning},
  pages={17827--17841},
  year={2023},
  organization={PMLR}
}

@article{krishnagopal2021spectral,
  title={Spectral Detection of Simplicial Communities via Hodge Laplacians},
  author={Krishnagopal, Sanjukta and Bianconi, Ginestra},
  journal={Physical Review E},
  volume={104},
  number={6},
  pages={064303},
  year={2021},
  publisher={American Physical Society}
}

@article{krishnagopal2022mountaineering,
  title={The Collective vs Individual Nature of Mountaineering: A Network and Simplicial Approach},
  author={Krishnagopal, Sanjukta},
  journal={Applied Network Science},
  volume={7},
  number={62},
  year={2022},
  publisher={Springer}
}

@article{battiston2020networks,
  title={Networks Beyond Pairwise Interactions: Structure and Dynamics},
  author={Battiston, Federico and Cencetti, Giulia and Iacopini, Iacopo and Latora, Vito and Lucas, Maxime and Patania, Alice and Young, Jean-Gabriel and Petri, Giovanni},
  journal={Physics Reports},
  volume={874},
  pages={1--92},
  year={2020},
  publisher={Elsevier}
}

@inproceedings{bodnar2021weisfeiler,
  title={Weisfeiler and Lehman Go Topological: Message Passing Simplicial Networks},
  author={Bodnar, Cristian and Frasca, Fabrizio and Wang, Yu Guang and Otter, Nina and Mont{\'u}far, Guido F. and Li{\`o}, Pietro and Bronstein, Michael M.},
  booktitle={Proceedings of the 38th International Conference on Machine Learning},
  series={Proceedings of Machine Learning Research},
  volume={139},
  pages={1026--1037},
  year={2021},
  publisher={PMLR}
}

@inproceedings{ebli2020simplicial,
  title={Simplicial Neural Networks},
  author={Ebli, Stefania and Defferrard, Micha{\"e}l and Spreemann, Gard},
  booktitle={NeurIPS Workshop on Topological Data Analysis and Beyond},
  year={2020}
}

@inproceedings{papamarkou2024topological,
  title={Position: Topological Deep Learning is the New Frontier for Relational Learning},
  author={Papamarkou, Theodore and Birdal, Tolga and Bronstein, Michael M. and Carlsson, Gunnar E. and Curry, Justin and Gao, Yue and Hajij, Mustafa and Kwitt, Roland and Lio, Pietro and Di Lorenzo, Paolo and others},
  booktitle={Proceedings of the 41st International Conference on Machine Learning},
  series={Proceedings of Machine Learning Research},
  volume={235},
  year={2024},
  publisher={PMLR}
}

@article{schaub2021signal,
  title={Signal Processing on Higher-Order Networks: Livin' on the Edge ... and Beyond},
  author={Schaub, Michael T. and Zhu, Yu and Seby, Jean-Baptiste and Roddenberry, T. Mitchell and Segarra, Santiago},
  journal={Signal Processing},
  volume={187},
  pages={108149},
  year={2021},
  doi={10.1016/j.sigpro.2021.108149}
}

@article{bick2023higher,
  title={What Are Higher-Order Networks?},
  author={Bick, Christian and Gross, Elizabeth and Harrington, Heather A. and Schaub, Michael T.},
  journal={SIAM Review},
  volume={65},
  number={3},
  pages={686--731},
  year={2023},
  doi={10.1137/21M1414024}
}

@article{su2024hodge,
  title   = {Hodge Decomposition of Single-Cell {RNA} Velocity},
  author  = {Su, Zhe and Tong, Yiying and Wei, Guo-Wei},
  journal = {Journal of Chemical Information and Modeling},
  volume  = {64},
  number  = {8},
  pages   = {3558--3568},
  year    = {2024},
  doi     = {10.1021/acs.jcim.4c00132}
}

@article{hajij2022topological,
  title={Topological Deep Learning: Going Beyond Graph Data},
  author={Hajij, Mustafa and Zamzmi, Ghada and Papamarkou, Theodore and Miolane, Nina and Guzm{\'a}n-S{\'a}enz, Aldo and Ramamurthy, Karthikeyan Natesan and Birdal, Tolga and Dey, Tamal K. and Mukherjee, Soham and Samaga, Shreyas N. and Livesay, Neal and Walters, Robin and Rosen, Paul and Schaub, Michael T.},
  journal={arXiv preprint arXiv:2206.00606},
  year={2022}
}

@article{barbarossa2020topological,
  title={Topological Signal Processing over Simplicial Complexes},
  author={Barbarossa, Sergio and Sardellitti, Stefania},
  journal={IEEE Transactions on Signal Processing},
  volume={68},
  pages={2992--3007},
  year={2020},
  publisher={IEEE},
  doi={10.1109/TSP.2020.2981920}
}

@misc{dblp,
  author = {{dblp team}},
  title = {{dblp computer science bibliography}},
  howpublished = {\url{https://dblp.org}},
  note = {Metadata released under CC0 1.0 Public Domain Dedication, with ODC-BY 1.0 as a secondary license; accessed 2026-04-28}
}

@article{schaub2020random,
  title={Random Walks on Simplicial Complexes and the Normalized Hodge 1-Laplacian},
  author={Schaub, Michael T. and Benson, Austin R. and Horn, Paul and Lippner, Gabor and Jadbabaie, Ali},
  journal={SIAM Review},
  volume={62},
  number={2},
  pages={353--391},
  year={2020},
  publisher={SIAM},
  doi={10.1137/18M1201019}
}

@article{yang2022simplicialfilters,
  title={Simplicial Convolutional Filters},
  author={Yang, Maosheng and Isufi, Elvin and Schaub, Michael T. and Leus, Geert},
  journal={IEEE Transactions on Signal Processing},
  volume={70},
  pages={4633--4648},
  year={2022},
  publisher={IEEE},
  doi={10.1109/TSP.2022.3207045}
}

@article{benson2018simplicial,
  title={Simplicial Closure and Higher-order Link Prediction},
  author={Benson, Austin R. and Abebe, Rediet and Schaub, Michael T. and Jadbabaie, Ali and Kleinberg, Jon},
  journal={Proceedings of the National Academy of Sciences},
  volume={115},
  number={48},
  pages={E11221--E11230},
  year={2018},
  publisher={National Academy of Sciences},
  doi={10.1073/pnas.1800683115}
}
\bibliographystyle{plainnat}
\appendix

\section{Derivation of the TopoNTK Recursion}
\label{app:ntk_derivation}

We derive the recursion used in Definition~\ref{def:topontk} from a standard
infinite-width NTK limit. The derivation is for the post-activation propagation
convention
\[
H_X^{(\ell+1)}
=
P_X\,\sigma\!\left(H_X^{(\ell)}W^{(\ell)}\right),
\qquad
P_X=\gamma I+\alpha L_\downarrow(X)+\beta L_\uparrow(X),
\]
where \(H_X^{(0)}\in\mathbb R^{|E_X|\times d}\) is the input edge-feature matrix,
\(\sigma\) is applied entrywise, and the weights are independent across layers with
\[
W^{(0)}_{ab}\sim N(0,1/d),
\qquad
W^{(\ell)}_{ab}\sim N(0,1/m_\ell)
\quad\text{for }\ell\ge 1.
\]
For notational simplicity we write all hidden widths as \(m_\ell\), and take the
standard infinite-width limit \(m_1,\ldots,m_L\to\infty\). The deterministic
operators \(P_X\) are fixed with respect to the trainable weights.

To obtain a scalar-output NTK for each edge, append an independent linear readout
\[
f_X(i)
=
\frac{1}{\sqrt{m_L}}\sum_{r=1}^{m_L} a_r H^{(L)}_{X,ir},
\qquad
a_r\sim N(0,1),
\]
with the same readout weights used when comparing two inputs. The edge-by-edge
NTK is
\[
\Theta^{(L)}(X,Y)_{ij}
=
\lim_{m\to\infty}
\left\langle
\nabla_\theta f_X(i),\nabla_\theta f_Y(j)
\right\rangle ,
\]
where \(\theta\) contains all hidden-layer weights and the readout weights. This
is the matrix-valued kernel acting on edge outputs. Pooled graph- or
complex-level kernels are obtained from this edge-level kernel by the pooling
operation in Definition~\ref{def:topontk}.

We first compute the covariance recursion. Define the empirical layer covariance
\[
\Sigma_m^{(\ell)}(X,Y)
=
\frac{1}{m_\ell}H_X^{(\ell)}(H_Y^{(\ell)})^\top
\in\mathbb R^{|E_X|\times |E_Y|},
\]
with the convention
\[
\Sigma^{(0)}(X,Y)=\frac{1}{d}H_X^{(0)}(H_Y^{(0)})^\top .
\]
For fixed edge indices \(i\in E_X\), \(j\in E_Y\), and hidden coordinate \(r\),
the preactivations
\[
Z^{(\ell)}_{X,ir}
=
\sum_{a=1}^{m_\ell} H^{(\ell)}_{X,ia}W^{(\ell)}_{ar},
\qquad
Z^{(\ell)}_{Y,jr}
=
\sum_{a=1}^{m_\ell} H^{(\ell)}_{Y,ja}W^{(\ell)}_{ar}
\]
are conditionally centered Gaussian in the infinite-width limit, with covariance
\[
\mathbb E\!\left[
Z^{(\ell)}_{X,ir} Z^{(\ell)}_{Y,jr}
\,\middle|\,
H_X^{(\ell)},H_Y^{(\ell)}
\right]
=
\Sigma_m^{(\ell)}(X,Y)_{ij}.
\]
Their variances are the corresponding diagonal entries of
\(\Sigma_m^{(\ell)}(X,X)\) and \(\Sigma_m^{(\ell)}(Y,Y)\). Hence the
post-activation covariance before Hodge propagation is the dual activation map
\[
\Phi(\Sigma^{(\ell)}(X,Y))_{ij}
=
\mathbb E[\sigma(g_i^X)\sigma(g_j^Y)],
\]
where \((g_i^X,g_j^Y)\) is the centered Gaussian pair with these two variances
and this cross-covariance. Applying the deterministic propagators on the two
sides gives
\[
\Sigma^{(\ell+1)}(X,Y)
=
P_X\,\Phi(\Sigma^{(\ell)}(X,Y))\,P_Y^\top .
\]
This proves the covariance recursion.

We now derive the tangent-kernel recursion. Let
\(\Theta^{(\ell)}(X,Y)\) denote the infinite-width tangent kernel associated with
the network truncated at layer \(\ell\), before the final readout contribution at
the next layer. The usual NTK chain rule gives two contributions when passing
from layer \(\ell\) to layer \(\ell+1\).

First, gradients with respect to parameters from previous layers are multiplied
by the activation derivatives at the two inputs. In the infinite-width limit,
averaging over hidden channels replaces the product of derivatives by the
derivative covariance map
\[
\dot\Phi(\Sigma^{(\ell)}(X,Y))_{ij}
=
\mathbb E[\sigma'(g_i^X)\sigma'(g_j^Y)].
\]
Since the deterministic propagation \(P_X\) is applied after the activation, the
old tangent kernel is transformed as
\[
P_X\Bigl[
\Theta^{(\ell)}(X,Y)
\odot
\dot\Phi(\Sigma^{(\ell)}(X,Y))
\Bigr]P_Y^\top .
\]
Here \(\odot\) denotes the Hadamard product in the standard edge-coordinate basis.

Second, the parameters \(W^{(\ell)}\) of the new layer contribute the covariance
of the new hidden representation, which is exactly
\(\Sigma^{(\ell+1)}(X,Y)\). Equivalently, this is the standard new-layer
contribution in the fully-connected NTK recursion, followed on the two sides by
the deterministic propagators.

Combining the old-parameter and new-parameter contributions yields
\[
\Theta^{(\ell+1)}(X,Y)
=
P_X\Bigl[
\Theta^{(\ell)}(X,Y)
\odot
\dot\Phi(\Sigma^{(\ell)}(X,Y))
\Bigr]P_Y^\top
+
\Sigma^{(\ell+1)}(X,Y),
\]
initialized by
\[
\Theta^{(0)}(X,Y)=\Sigma^{(0)}(X,Y).
\]
Thus the infinite-width edge-level NTK satisfies exactly the recursion stated in
Definition~\ref{def:topontk}.

For ReLU, \(\Phi\) and \(\dot\Phi\) are the standard arc-cosine covariance maps,
with the derivative map understood on covariance pairs whose self-variances are
positive. The finite-depth stability result therefore assumes that the relevant
self-covariance diagonals stay bounded away from zero. The derivation above also
makes clear why the Hodge propagators \(P_X,P_Y\) appear outside the dual
activation maps for the post-activation convention. A pre-activation convention
\(\sigma(P_XH_X^{(\ell)}W^{(\ell)})\) would instead place the Hodge propagation
inside the covariance map and gives a different, though closely related, kernel.

\section{Proof of Proposition~\ref{prop:filled-simplex-sensitivity}}
\label{app:expressivity}
\begin{proof}
The graph-NTK statement follows immediately. If the two complexes have the same
oriented \(1\)-skeleton, then every graph-derived propagation operator is the
same. If the graph NTK uses only features determined by that \(1\)-skeleton,
then the input covariance is also the same. Therefore the graph NTK recursion
produces the same kernel for \(X\) and \(X'\).

For TopoNTK, the only difference between the two propagation operators is the
upper Hodge term:
\[
P_X-P_{X'}
=
\beta\bigl(L_\uparrow(X)-L_\uparrow(X')\bigr).
\]
Thus, if \(\beta>0\) and \(L_\uparrow(X)\neq L_\uparrow(X')\), then
\(P_X\neq P_{X'}\). In the linear-activation case, the first covariance update is
\[
\Sigma_X^{(1)}
=
P_X C P_X^\top,
\qquad
\Sigma_{X'}^{(1)}
=
P_{X'} C P_{X'}^\top,
\]
where \(C=\Sigma^{(0)}\) is the input edge-feature covariance. Hence, under the
stated condition,
\[
\Sigma_X^{(1)}\neq \Sigma_{X'}^{(1)}.
\]
Since the NTK recursion contains this covariance as an additive contribution,
the edge-level TopoNTK differs after one layer.

Finally, a pooled scalar kernel applies a linear pooling functional to the
edge-level kernel. Therefore edge-level separation implies pooled separation
whenever the pooling functional is nonzero on the resulting kernel difference.
For sum pooling this condition is
\[
\mathbf 1^\top
\bigl(\Theta_X^{(L)}-\Theta_{X'}^{(L)}\bigr)
\mathbf 1
\neq 0.
\]
\end{proof}

\section{Proof of Hodge-Compatible Invariance}
\label{app:hodge_invariance_NTK}
\begin{proof}[Proof of Proposition~\ref{prop:hodge_compatibility}]
By Proposition~\ref{prop:hodge}, the propagation operator
\(P_{\gamma,\alpha,\beta}\) is block diagonal with respect to
\[
C^1(X)=\mathcal{E}\oplus\mathcal{H}\oplus\mathcal{C}.
\]
By assumption, \(\Sigma^{(0)}(X,X)\) is block diagonal in this decomposition, and the finite-depth covariance and tangent-kernel updates preserve block diagonal structure. Inducting over layers gives that \(\Sigma^{(\ell)}\) and \(\Theta^{(\ell)}\) are block diagonal for all \(\ell\le L\). Hence \(K_X=\Theta^{(L)}(X,X)\) is block diagonal and maps each of \(\mathcal E\), \(\mathcal H\), and \(\mathcal C\) into itself.

For nonlinear NTK recursions involving entrywise Hadamard products, the preservation of Hodge block structure is an additional compatibility condition, not a generic consequence of orthogonal Hodge block diagonal structure. This is why the main text distinguishes exact Hodge preservation of the propagation operator from Hodge-aligned bias of the standard nonlinear TopoNTK.
\end{proof}

\section{Proof of Spectral Learning}
\label{app:spectral}

\begin{proof}[Proof of Theorem~\ref{th:spectral}]
Let $K_X$ be the within-complex TopoNTK. Since $K_X$ is symmetric positive
semidefinite, it admits an orthonormal eigendecomposition
\[
K_X u_j = \kappa_j u_j,
\qquad
\kappa_j \ge 0,
\]
with $\{u_j\}$ an orthonormal basis for the relevant edge-signal space.

Consider squared-loss kernel gradient flow with target $y$ and prediction
$f_t$:
\[
\frac{d f_t}{dt}
=
-K_X(f_t-y),
\qquad
f_0=0.
\]
Define the error
\[
e_t=f_t-y.
\]
Then
\[
\frac{d e_t}{dt}
=
\frac{d f_t}{dt}
=
-K_X e_t.
\]
Expand the error in the eigenbasis of $K_X$:
\[
e_t=\sum_j e_j(t)u_j,
\qquad
e_j(t)=\langle e_t,u_j\rangle .
\]
Using $K_Xu_j=\kappa_j u_j$, we obtain
\[
\frac{d e_t}{dt}
=
-\sum_j \kappa_j e_j(t)u_j.
\]
Equating coefficients in the orthonormal basis gives the scalar ODE
\[
\frac{d e_j(t)}{dt}
=
-\kappa_j e_j(t).
\]
Thus
\[
e_j(t)=e_j(0)e^{-\kappa_j t}.
\]
Since $f_0=0$, we have $e_0=-y$, and therefore
\[
e_j(0)
=
\langle e_0,u_j\rangle
=
-\langle y,u_j\rangle .
\]
Hence
\[
e_t
=
-\sum_j e^{-\kappa_j t}\langle y,u_j\rangle u_j.
\]
Because $f_t=y+e_t$, and
\[
y=\sum_j \langle y,u_j\rangle u_j,
\]
we obtain
\[
f_t
=
\sum_j
\left(1-e^{-\kappa_j t}\right)
\langle y,u_j\rangle u_j.
\]

If the Hodge-compatibility assumptions of Proposition~\ref{prop:hodge_compatibility}
hold, then $K_X$ is block diagonal with respect to
\[
C^1(X)=\mathcal{E}\oplus\mathcal{H}\oplus\mathcal{C}.
\]
Therefore each block admits an orthonormal eigenbasis, and the eigenvectors
$\{u_j\}$ may be chosen within the exact, harmonic, and coexact subspaces. In
that case, the expansion above shows that learning proceeds independently along
Hodge-aligned kernel eigenmodes, with mode $u_j$ learned at rate $\kappa_j$.
\end{proof}

\section{Proof of the Stability Theorem}
\label{app:stability}
We prove the stability result in Theorem~\ref{th:stability}. 
\begin{proof}
Let $X$ and $X'$ be simplicial complexes with the same $1$-skeleton and different triangle sets. 
Since the $1$-skeleton is fixed, the lower Hodge Laplacian is unchanged:
\[
L_{\downarrow}(X)=L_{\downarrow}(X').
\]
Thus all variation in the TopoNTK comes from the upper Hodge Laplacian
\[
L_{\uparrow}(X)=B_2B_2^\top.
\]

Let
\[
\Delta B_2 := B_2-B_2',
\qquad
\Delta L_{\uparrow}:=L_{\uparrow}(X)-L_{\uparrow}(X').
\]
Then
\[
\Delta L_{\uparrow}
=
B_2B_2^\top-B_2'B_2'^{\top}.
\]
Adding and subtracting $B_2'B_2^\top$ gives
\[
\Delta L_{\uparrow}
=
(B_2-B_2')B_2^\top
+
B_2'(B_2-B_2')^\top.
\]
Therefore, applying the triangle inequality
\[
\|\Delta L_{\uparrow}\|
\leq
\|B_2-B_2'\|\,\|B_2\|
+
\|B_2'\|\,\|B_2-B_2'\|.
\]
Hence
\[
\|\Delta L_{\uparrow}\|
\leq
(\|B_2\|+\|B_2'\|)
\|B_2-B_2'\|.
\]
Assuming $\|B_2\|,\|B_2'\|\leq M$, we obtain
\[
\|\Delta L_{\uparrow}\|
\leq
2M\|B_2-B_2'\|.
\]

Now define the propagation operators
\[
P_X=\gamma I+\alpha L_\downarrow(X)+\beta L_\uparrow(X),\qquad
P_{X'}=\gamma I+\alpha L_\downarrow(X')+\beta L_\uparrow(X').
\]
The residual term cancels and the lower term is unchanged because the
1-skeleton is fixed, so
\[
P_X-P_{X'}
=
\beta\bigl(L_{\uparrow}(X)-L_{\uparrow}(X')\bigr).
\]
Thus
\[
\|P_X-P_{X'}\|
\leq
\beta\|L_\uparrow-L_\uparrow'\|,
\]
and, if \(\|B_2\|,\|B_2'\|\le M\), also \(\|P_X-P_{X'}\|\le 2\beta M\|B_2-B_2'\|\).

We next show that the TopoNTK recursion is Lipschitz in the propagation operator. 
Assume that all covariance matrices remain in a bounded set and that the activation covariance maps $\Phi$ and $\dot{\Phi}$ are Lipschitz on this set:
\[
\|\Phi(A)-\Phi(B)\|\leq L_{\Phi}\|A-B\|,
\qquad
\|\dot{\Phi}(A)-\dot{\Phi}(B)\|\leq L_{\dot{\Phi}}\|A-B\|.
\]
Also assume
\[
\|P_X\|,\|P_{X'}\|\leq R.
\]

Let $\Sigma_X^{(\ell)}$ and $\Sigma_{X'}^{(\ell)}$ denote the depth-$\ell$ covariance matrices for $X$ and $X'$. 
The recursion is
\[
\Sigma_X^{(\ell+1)}
=
P_X\Phi(\Sigma_X^{(\ell)})P_X^\top,
\qquad
\Sigma_{X'}^{(\ell+1)}
=
P_{X'}\Phi(\Sigma_{X'}^{(\ell)})P_{X'}^\top.
\]
Subtracting gives
\[
\Sigma_X^{(\ell+1)}-\Sigma_{X'}^{(\ell+1)}
=
P_X\Phi(\Sigma_X^{(\ell)})P_X^\top
-
P_{X'}\Phi(\Sigma_{X'}^{(\ell)})P_{X'}^\top.
\]
Add and subtract intermediate terms:
\[
\begin{aligned}
\Sigma_X^{(\ell+1)}-\Sigma_{X'}^{(\ell+1)}
={}&
(P_X-P_{X'})\Phi(\Sigma_X^{(\ell)})P_X^\top \\
&+
P_{X'}\bigl(\Phi(\Sigma_X^{(\ell)})-\Phi(\Sigma_{X'}^{(\ell)})\bigr)P_X^\top \\
&+
P_{X'}\Phi(\Sigma_{X'}^{(\ell)})(P_X-P_{X'})^\top .
\end{aligned}
\]
Taking norms and using boundedness,
\[
\|\Sigma_X^{(\ell+1)}-\Sigma_{X'}^{(\ell+1)}\|
\leq
C_{\ell}
\|P_X-P_{X'}\|
+
R^2L_{\Phi}
\|\Sigma_X^{(\ell)}-\Sigma_{X'}^{(\ell)}\|.
\]
By induction over finite depth $L$, there exists a constant $C_{\Sigma,L}$ such that
\[
\|\Sigma_X^{(L)}-\Sigma_{X'}^{(L)}\|
\leq
C_{\Sigma,L}\|P_X-P_{X'}\|.
\]

The tangent kernel recursion is
\[
\Theta_X^{(\ell+1)}
=
P_X\left[\Theta_X^{(\ell)}\odot\dot\Phi(\Sigma_X^{(\ell)})\right]P_X^\top
+
\Sigma_X^{(\ell+1)}.
\]
Using the same Lipschitz and boundedness assumptions, and the covariance bound above, another induction gives
\[
\|\Theta_X^{(L)}-\Theta_{X'}^{(L)}\|
\leq
C_{\Theta,L}\|P_X-P_{X'}\|.
\]
Since $K_X=\Theta_X^{(L)}$ and $K_{X'}=\Theta_{X'}^{(L)}$, we obtain
\[
\|K_X-K_{X'}\|
\leq
C_{\Theta,L}\|P_X-P_{X'}\|.
\]
Combining with the propagation bound,
\[
\|K_X-K_{X'}\|
\leq
C\|L_\uparrow-L_\uparrow'\|,
\]
with the displayed \(B_2\) bound as a corollary under \(\|B_2\|,\|B_2'\|\le M\).

It remains to prove the prediction stability bound. 
Kernel ridge regression gives
\[
\widehat{y}_X
=
K_X(K_X+\lambda I)^{-1}y.
\]
Using the identity
\[
K(K+\lambda I)^{-1}
=
I-\lambda(K+\lambda I)^{-1},
\]
we have
\[
\widehat{y}_X-\widehat{y}_{X'}
=
\lambda
\left[
(K_{X'}+\lambda I)^{-1}
-
(K_X+\lambda I)^{-1}
\right]y.
\]
By the resolvent identity,
\[
(K_{X'}+\lambda I)^{-1}
-
(K_X+\lambda I)^{-1}
=
(K_{X'}+\lambda I)^{-1}
(K_X-K_{X'})
(K_X+\lambda I)^{-1}.
\]
Therefore,
\[
\|\widehat{y}_X-\widehat{y}_{X'}\|
\leq
\lambda
\|(K_{X'}+\lambda I)^{-1}\|
\|K_X-K_{X'}\|
\|(K_X+\lambda I)^{-1}\|
\|y\|.
\]
Since $K_X$ and $K_{X'}$ are positive semidefinite,
\[
\|(K_X+\lambda I)^{-1}\|\leq \frac{1}{\lambda},
\qquad
\|(K_{X'}+\lambda I)^{-1}\|\leq \frac{1}{\lambda}.
\]
Thus
\[
\|\widehat{y}_X-\widehat{y}_{X'}\|
\leq
\frac{1}{\lambda}
\|K_X-K_{X'}\|\|y\|.
\]
Using the kernel stability bound,
\[
\|\widehat{y}_X-\widehat{y}_{X'}\|
\leq
\frac{C}{\lambda}
\|L_\uparrow-L_\uparrow'\|\|y\|,
\]
again with the \(B_2\) version following from \(\|L_\uparrow-L_\uparrow'\|\le 2M\|B_2-B_2'\|\). This completes the proof.
    
\end{proof}

\end{document}